\documentclass{article}
\usepackage{graphicx, amsmath, amssymb, enumitem, threeparttable} 
\usepackage[margin=1in]{geometry}
\usepackage{hyperref} 
\usepackage[round]{natbib} 
\usepackage{booktabs} 
\usepackage{amsthm}

\newtheorem{theorem}{Theorem}
\newtheorem{lemma}[theorem]{Lemma}
\newtheorem{corollary}[theorem]{Corollary}

\DeclareMathOperator{\Tr}{tr}

\newcommand{\CtwoG}{C_2(g)}
\newcommand{\ConeRK}{C_1(r,K)}
\newcommand{\Cg}{C(g)} 

\title{Benign Overfitting in Linear Classifiers with a Bias Term}

\author{Yuta Kondo\\
University of Toronto\\
\texttt{yuta.kondo@mail.utoronto.ca}}
\date{}

\begin{document}

\maketitle

\begin{abstract}
Overparameterized models often generalize well even when they interpolate noisy training data. This is known as benign overfitting. For linear classification, \citet{Hashimoto2025} analyzed the phenomenon under a broad class of mixture distributions, but only for homogeneous classifiers without a bias term. We extend their framework to classifiers with an intercept. Benign overfitting still occurs, but the intercept perturbs the normalized Gram matrix of the noise and creates extra constraints on the covariance. These constraints are strongest with label noise. Their effect depends on the covariance: under isotropic noise they are dominated by the homogeneous conditions, while in anisotropic or noisy regimes they can raise the dimensionality needed for benign generalization. Thus the bias term changes the theory in some covariance regimes and leaves the asymptotic thresholds unchanged in others.
\end{abstract}

\section{Introduction}
Highly overparameterized models can interpolate noisy training data and still generalize well \citep{Zhang2017, Belkin2019}. This behavior, called benign overfitting, has been studied in linear regression \citep{Bartlett2020, Hastie2022}, kernel methods \citep{Liang2020}, and linear classification. In classification, much of the theory concerns the maximum margin classifier. For separable data it is the direction obtained by running gradient descent on logistic or exponential loss \citep{Soudry2018}, so its generalization properties are a natural benchmark for interpolating linear methods. Existing guarantees often assume Gaussian or sub-Gaussian mixtures \citep{Chatterji2021, Cao2021, Wang2022}.

\citet{Hashimoto2025} weakened these assumptions substantially. Working with a broad non-sub-Gaussian mixture model (Model EM below), they obtained test-error bounds and identified phase transitions without requiring sub-Gaussianity or isotropy. Like most of this literature, however, their analysis is for \textit{homogeneous} classifiers, i.e., without a bias term. In practice one almost always includes an intercept. \citet{Hashimoto2025} note that the inhomogeneous case can be reduced to the homogeneous one by appending a constant feature, but they leave open what that augmentation does to the conditions for benign overfitting.

We carry out that extension for both noiseless and noisy labels. Adding a bias perturbs the normalized Gram matrix of the features. The shared constant coordinate increases off-diagonal correlations among the noise vectors, and keeping those correlations small enough for the Hashimoto et al.\ argument forces stronger conditions on the effective dimension $\Tr(\Sigma)$ relative to the sample size $n$. The new constraints are most demanding when label noise is present, because the allowed Gram-matrix error must then scale with the noise rate $\eta$. At the same time, the effect is covariance-dependent: under isotropic noise the new terms are dominated by conditions already present in the homogeneous theory, so the asymptotic threshold for benign overfitting is unchanged.

Our contributions are as follows.
\begin{enumerate}
    \item We prove benign overfitting for the maximum margin classifier with a bias term under Model EM, with and without label noise.
    \item We identify the new conditions on $\Tr(\Sigma)$ and $n$ coming from the intercept, and determine when they dominate the homogeneous requirements.
    \item We prove auxiliary bounds (Appendix~\ref{sec:new_events}) controlling the bias-induced perturbations of the Gram matrix and related concentration events.
\end{enumerate}

\section{Preliminaries and Problem Setup}
We use the notation of \citet{Hashimoto2025} and work with their extended non-sub-Gaussian mixture model (EM).

\subsection{Model (EM)}
The observations consist of $n$ i.i.d copies ($\boldsymbol{x}_i, y_{N,i}$) of the pair ($\boldsymbol{x}, y_N$). Here, $y\in\{-1, 1\}$ is a random variable satisfying $P(y=1)=P(y=-1)=1/2$, and for a deterministic $\boldsymbol{\mu}\in\mathbb{R}^p$ we have
\[
\boldsymbol{x} = y\boldsymbol{\mu}+\boldsymbol{z}
\]
where $\boldsymbol{z}=g\Sigma^{1/2}\boldsymbol{\xi}$, characterized by:
\begin{enumerate}
    \item $g\in(0, \infty)$ such that, for some $l\in[2, \infty]$ and $k\in(2, 4]$, $\mathbb{E}g^2=1, \mathbb{E}g^l < \infty, \mathbb{E}g^{-k} < \infty$.
    \item $\Sigma \in \mathbb{R}^{p \times p}$ is a deterministic positive semidefinite matrix.
    \item $\boldsymbol{\xi}\in\mathbb{R}^p$ with independent entries, independent of $g$ such that, for some $r\in(2, 4]$ and $K>0$, $\mathbb{E}\xi_j=0, \mathbb{E}\xi_j^2=1, \mathbb{E}|\xi_j|^r \le K$, $\forall j = 1, ..., p$.
\end{enumerate}
For $\eta\in[0, 1/2)$, the observed label $y_N$ is generated such that $y_N=-y$ w.p. $\eta$ (noisy) and $y_N=y$ w.p. $1-\eta$ (clean).

\subsection{The Inhomogeneous Model}
We analyze inhomogeneous linear classifiers $sgn(\langle \boldsymbol{w}, \boldsymbol{x} \rangle + b)$. This is equivalent to a homogeneous classifier $sgn(\langle \tilde{\boldsymbol{w}}, \tilde{\boldsymbol{x}} \rangle)$ applied to the extended vectors:
\[
\tilde{\boldsymbol{x}}=(\boldsymbol{x}, 1), \quad \tilde{\boldsymbol{z}}=(\boldsymbol{z}, 1), \quad \tilde{\boldsymbol{\mu}}=(\boldsymbol{\mu}, 0).
\]
We study the maximum margin classifier $\hat{\tilde{\boldsymbol{w}}}$ on these extended features. For separable data this is again the gradient-descent limit for logistic or exponential loss \citep{Soudry2018}; the only change relative to Hashimoto et al.\ is that the feature vectors live in $\mathbb{R}^{p+1}$.

\subsection{Notations and Key Events}
We use the standard matrix notation ($X, Z, \boldsymbol{y}, \boldsymbol{y}_N$). When $\boldsymbol{z}_i \neq \boldsymbol{0}$, $\Delta(z)$ is the diagonal matrix of norms $\|\boldsymbol{z}_i\|$. Normalized quantities are written with a check: $\check{Z}=\Delta(z)^{-1}Z$.

The homogeneous analysis of \citet{Hashimoto2025} is organized around high-probability events $E_1(\varepsilon),\ldots,E_5(\gamma, \rho)$, where $\rho = \mathbb{E}[g^{-2}]/\Tr(\Sigma)$.

In the inhomogeneous model we use the analogous events for $\check{\tilde{Z}} = \Delta(\tilde{z})^{-1}\tilde{Z}$:
\begin{align}
\tilde{E}_1(\tilde{\varepsilon}) &:= \big\{\|\check{\tilde{Z}}\check{\tilde{Z}}^\top - I_n\|\leq \tilde{\varepsilon}\big\}, \quad \text{(Near-orthogonality of normalized noise)} \label{eq:E1}
\\
\tilde{E}_2(\tilde{\alpha}_2, \tilde{\alpha}_\infty) &:= \big\{\|\check{\tilde{Z}}\tilde{\boldsymbol{\mu}}\| \leq \tilde{\alpha}_2 \|\tilde{\boldsymbol{\mu}}\| \text{ and } \|\check{\tilde{Z}} \tilde{\boldsymbol{\mu}}\|_\infty \leq \tilde{\alpha}_\infty \|\tilde{\boldsymbol{\mu}}\|\big\}, \quad \text{(Alignment of noise with signal)} \label{eq:E2}
\\
\tilde{E}_3(\tilde{M}) &:= \big\{\max_{i} \|\tilde{\boldsymbol{z}}_i\|\leq \tilde{M}\big\}, \quad \text{(Bound on maximum noise norm)} \label{eq:E3}
\\
\tilde{E}_4(\tilde{\beta}, \rho) & := \Big\{ \Big|\dfrac{1}{n}\sum_{i} \dfrac{1}{\|\tilde{\boldsymbol{z}}_i\|^{2}}-\rho \Big|\leq \tilde{\beta}\rho  \Big\}, \quad \text{(Concentration of inverse noise norms)} \label{eq:E4}
\\
\tilde{E}_5(\tilde{\gamma}, \rho) & := \Big\{\Big|\dfrac{1}{n}\sum_{i} \dfrac{y_{N, i} y_i}{\|\tilde{\boldsymbol{z}}_i\|^2} -(1-2\eta)\rho \Big|\leq \tilde{\gamma}\rho \Big\}. \quad \text{(Concentration involving labels)} \label{eq:E5}
\end{align}

\textit{Remark on $\rho$:} In $\tilde{E}_4$ and $\tilde{E}_5$ we still concentrate around the homogeneous expectation $\rho$, not around $\mathbb{E}[1/\|\tilde{\boldsymbol{z}}\|^2]$. This lets us reuse the arguments of \citet{Hashimoto2025}. When $\Tr(\Sigma)$ is large, $\rho$ is a good proxy for the inhomogeneous mean; the parameters $\tilde{\beta}$ and $\tilde{\gamma}$ absorb both concentration error and the mean shift (Appendix~\ref{sec:new_events}).

Appendix~\ref{sec:new_events} relates the tilde events to the original events $E_i$.

\section{Main Results}
We now state test-error bounds for the inhomogeneous maximum margin classifier, obtained by checking the hypotheses of Theorems 3.2 and 3.5 in \citet{Hashimoto2025}.

\subsection{Noiseless Case}

We begin with the small-to-intermediate signal regime. Relative to the homogeneous case, we need additional control on $\Tr(\Sigma)$, recorded as $T_{Inhom}$, to keep the bias-induced perturbations small.

\begin{theorem}[Inhomogeneous, Noiseless, Intermediate Signal]
\label{thm:1}
Consider the inhomogeneous model with $\tilde{\boldsymbol{z}}=(\boldsymbol{z},1)$, where $\boldsymbol{z}$ is generated according to Model (EM) in the noiseless case ($\eta=0$). Assume $n \ge  \left(6\CtwoG\right)^\frac{k}{k-2} \delta^{-\frac{2}{k-2}}$.

Then, there exists a constant $C$ (depending only on the distribution of $g, \xi$) such that, provided:
\[
\|\mu\|^2 \ge C \delta^{-1/2}\|\Sigma^{1/2}\mu\|
\]
and that $\Tr(\Sigma) \ge C \cdot \max\{T_{Hom}, T_{Inhom}\}$, where
\begin{align*}
T_{Hom} &= \left(\frac{n}{\delta}\right)^{1/l} \max \left\{\left(\frac{n}{\delta}\right)^{2/r}n^{1/2}\|\Sigma\|_F\max\{p^{2/r-1/2}, n^{2/r}\}, \sqrt{\frac{n}{\delta}}n\|\Sigma^{1/2}\mu\|\right\}, \\
T_{Inhom} &= \left(\frac{n}{\delta}\right)^{2/k+1/l}n^{3/2},
\end{align*}
we have, with a probability of at least $1-5\delta$, the gradient descent iterate for the inhomogeneous classifier converges in direction to the maximum margin classifier $\hat{\tilde{\boldsymbol{w}}}$, and the following test error bound holds for a universal constant $c$:
\[
\mathbb{P}_{(\boldsymbol{x},y)}(\langle \hat{\tilde{\boldsymbol{w}}}, y\tilde{\boldsymbol{x}} \rangle < 0) \le c\|\mathbb{E}[\tilde{\boldsymbol{z}}\tilde{\boldsymbol{z}}^{\top}]\| \left(\frac{1}{\|\tilde{\boldsymbol{\mu}}\|^2} + \frac{1}{n\rho\|\tilde{\boldsymbol{\mu}}\|^4}\right).
\]
\end{theorem}

\begin{proof}[Proof Sketch for Theorem \ref{thm:1}]
We check the hypotheses of Theorem 3.2(i) of \citet{Hashimoto2025} for the tilde variables and events. Appendix~\ref{sec:new_events} shows that $\tilde{E}_2$ and $\tilde{E}_3$ follow from $E_2$ and $E_3$, while $\tilde{E}_1$, $\tilde{E}_4$, and $\tilde{E}_5$ pick up extra terms from the constant coordinate in $\tilde{\boldsymbol{z}}$.

Under the stated assumptions those extras are small:
1. $\tilde{\beta} < 1/2$. Bound $\beta$ and the perturbation $\beta'$ from Lemma~\ref{lem:E4E5}. The assumption on $n$ controls $\beta$; making $\beta'$ small needs $\Tr(\Sigma) \succsim (n/\delta)^{4/k}$, which for large $n$ is weaker than the condition in step 3 (Appendix~\ref{sec:proof_thm1}).
2. The signal-strength hypotheses follow from the homogeneous assumptions ($T_{Hom}$), up to adjusting constants.
3. For $\tilde{\varepsilon}$ and $\tilde{M}$, use Lemma~\ref{lem:E1}. The strongest new requirement, $\Tr(\Sigma) \succsim (n/\delta)^{2/k+1/l}n^{3/2}$, comes from the perturbation in $\tilde{\varepsilon}$ in condition (iv) of Theorem 3.2(i).

The failure probability is at most $5\delta$. Full details are in Appendix~\ref{sec:proof_thm1}.
\end{proof}

Next, we consider the large signal regime.

\begin{theorem}[Inhomogeneous, Noiseless, Large Signal]
\label{thm:2}
Consider the inhomogeneous model with $\tilde{\boldsymbol{z}}=(\boldsymbol{z}, 1)$, where $\boldsymbol{z}$ is generated according to Model (EM) with $\eta=0$. Let $C_{H}>2$ be the constant required by Theorem 3.2(ii) of \citet{Hashimoto2025}. There exists a constant $\Cg$ depending only on the distribution of $g$. For any $\delta\in(0, \frac{1}{2})$, if the following conditions hold:
\begin{enumerate}
    \item (Large Signal) $\|\boldsymbol{\mu}\| \ge \frac{3\sqrt{2}}{2}C_{H}\|g\|_{L^l}\left(\frac{n}{\delta}\right)^{1/l}\sqrt{\mathrm{Tr}(\Sigma)}$
    \item (High Dimension) $\Tr(\Sigma) \ge 2\ConeRK\left(\frac{n}{\delta}\right)^{2/r}\max\{p^{2/r-1/2}, n^{2/r}\}\|\Sigma\|_F$
\end{enumerate}
then we have, with a probability of at least $1-2\delta$, the gradient descent iterate converges in direction to $\hat{\tilde{\boldsymbol{w}}}$, and the following test error bound holds for a universal constant $c$:
\[
\mathbb{P}_{(\boldsymbol{x},y)}(\langle \hat{\tilde{\boldsymbol{w}}}, y\boldsymbol{\tilde{x}} \rangle < 0) \le c\frac{\|\mathbb{E}[\tilde{\boldsymbol{z}}\tilde{\boldsymbol{z}}^{\top}]\|}{\|\tilde{\boldsymbol{\mu}}\|^2} = c\frac{\max(\|\Sigma\|, 1)}{\|\boldsymbol{\mu}\|^2}.
\]
\end{theorem}

\begin{proof}[Proof Sketch for Theorem \ref{thm:2}]
We apply Theorem 3.2(ii) of \citet{Hashimoto2025} to the inhomogeneous variables. The key condition is $\|\tilde{\boldsymbol{\mu}}\| \ge C_H\tilde{M}$. The High Dimension condition ensures $\varepsilon \le 1/2$ and that $E_3(M)$ holds with high probability. This allows us to apply Lemma~\ref{lem:M_bound} (Appendix~\ref{sec:aux_lemmas}), showing $M\ge 1$ (for sufficiently large $n$). Thus, by Lemma~\ref{lem:E3}, $\tilde{M} = \sqrt{M^2+1} \le \sqrt{2}M$. A sufficient condition is $\|\boldsymbol{\mu}\| \ge C_H\sqrt{2}M$. Substituting the definition of $M$ and using $\varepsilon \le 1/2$, we find that the Large Signal condition of the theorem is sufficient. The detailed proof is in Appendix~\ref{sec:proof_thm2}.
\end{proof}

\subsection{Noisy Case}

We now consider the case with label noise ($\eta > 0$).

\begin{theorem}[Inhomogeneous, Noisy]
\label{thm:3}
Consider the inhomogeneous model with $\tilde{\boldsymbol{z}}=(\boldsymbol{z}, 1)$, where $\boldsymbol{z}$ is generated according to Model (EM) with $\eta\in(0, \frac{1}{2})$. There exists a constant $C$ depending only on the distribution of $g, \boldsymbol{\xi}$ such that, if
\[
\|\boldsymbol{\mu}\|^2 \ge C\max\{\frac{1}{\eta}, \frac{1}{1-2\eta}\}\delta^{-1/2}\|\Sigma^{1/2}\boldsymbol{\mu}\|,
\]
and $\Tr(\Sigma) \ge C \cdot \max\{T'_{Hom}, T'_{Inhom}\}$, where
\begin{align*}
T'_{Hom} &= \frac{1}{\eta}\left(\frac{n}{\delta}\right)^{2/r+1/l}n^{1/2} \max \{p^{2/r-1/2}, n^{2/r}\}\|\Sigma\|_F, \\
T'_{Inhom} &= \frac{1}{\eta}\left(\frac{n}{\delta}\right)^{2/k+1/l}n^{3/2}
\end{align*}

together with one of the following holds:
\begin{enumerate}
    \item $\Tr(\Sigma) \ge C \left(\frac{n}{\delta}\right)^{1/2+1/l}n\|\Sigma^{1/2}\boldsymbol{\mu}\|$,
    \item $\|\boldsymbol{\mu}\|^2 \ge \frac{C}{\min(\eta, 1-2\eta)}\left(\frac{n}{\delta}\right)^{1/2+1/l}\|\Sigma^{1/2}\boldsymbol{\mu}\| \quad \text{and} \quad \Tr(\Sigma) \ge \frac{C}{\min(\eta, 1-2\eta)}\left(\frac{n}{\delta}\right)^{1+1/l}n^{1/2}\frac{\|\Sigma^{1/2}{\boldsymbol{\mu}}\|^2}{\|\boldsymbol{\mu}\|^2}$
\end{enumerate}
then, for $n \ge \delta^{-\frac{2}{k-2}}\left(\frac{32\CtwoG}{\min\{\eta, 1-2\eta\}}\right)^\frac{k}{k-2}$, with probability at least $1-6\delta$, the gradient descent iterate converges in direction to $\hat{\tilde{\boldsymbol{w}}}$, and the following test error bound holds:
\[
\mathbb{P}_{(\boldsymbol{x},y)}(\langle\hat{\tilde{\boldsymbol{w}}}, y\boldsymbol{\tilde{x}}\rangle < 0) \le \eta + \frac{c\|\mathbb{E}[\tilde{\boldsymbol{z}}\tilde{\boldsymbol{z}}^{\top}]\|}{(1-2\eta)^2}\left(\eta n \rho + \frac{1}{\|\boldsymbol{\tilde{\mu}}\|^2} + \frac{1}{n\rho\|\boldsymbol{\tilde{\mu}}\|^4}\right).
\]
\end{theorem}

\begin{proof}[Proof Sketch for Theorem \ref{thm:3}]
We check Theorem 3.5 of \citet{Hashimoto2025}: $\tilde{\varepsilon}$, $\tilde{\beta}$, and $\tilde{\gamma}$ must be small relative to $\eta$, and condition $(N_C)$ must hold.

1. Need $\tilde{\beta}, \tilde{\gamma} \le \frac{\min\{\eta, 1-2\eta\}}{8}$. The lower bound on $n$ controls $\beta$; the bound on $\Tr(\Sigma)$ controls $\beta'$ (Lemma~\ref{lem:E4E5}). The latter is weaker than the condition in step 3 for large $n$.

2. Need $\tilde{\varepsilon} \le \frac{\min\{\eta, 1-2\eta\}}{8}$. Lemma~\ref{lem:E1} gives a requirement of the form $\Tr(\Sigma) \succsim n(n/\delta)^{2/k}/\min\{\eta, 1-2\eta\}$, again weaker than step 3.

3. Condition $(N_C)$ includes $\tilde{\varepsilon} \tilde{M} \sqrt{n\rho} \le \eta/2$. This is where the leading new term $n^{3/2}(n/\delta)^{1/l+2/k}$ appears, through the perturbation $T$ in Lemma~\ref{lem:E1}. The rest of $(N_C)$ follows from $T'_{Hom}$ and the stated signal regimes.

See Appendix~\ref{sec:proof_thm3}.
\end{proof}

\subsection{Isotropic Case}

\begin{corollary}
\label{cor:1}
Consider the inhomogeneous model with $\tilde{\boldsymbol{z}}=(\boldsymbol{z}, 1)$, $\eta \in (0, \frac{1}{2})$, and assume that $\Sigma=I_p$. Then a sufficient condition for benign overfitting is that
\[
\|\mu\| \gg \left(\frac{p}{n}\right)^{1/4} \quad \text{and} \quad p \succsim \max\left\{n^\frac{4+(1+2/l)r}{2(r-2)}, n^{1+8/r+2/l}\right\}.
\]
\end{corollary}

\begin{proof}[Proof of Corollary \ref{cor:1}]
We verify the conditions of Theorem 3 in the isotropic case, where $\Sigma=I_p$ ($\Tr(\Sigma)=p$, $\|\Sigma\|_F=\sqrt{p}$).

The test error bound simplifies (using $\rho \approx \mathbb{E}[g^{-2}]/p$) to:
\[
\mathbb{P}(\langle\hat{\tilde{\boldsymbol{w}}}, y\boldsymbol{\tilde{x}}\rangle < 0) - \eta \le c'\left(\eta \frac{n}{p} + \frac{1}{\|\boldsymbol{\mu}\|^2} + \frac{p}{n\|\boldsymbol{\mu}\|^4}\right).
\]
For benign overfitting, the RHS must vanish, requiring $n/p \to 0$ and $\|\boldsymbol{\mu}\| \gg (p/n)^{1/4}$.

We verify the conditions of Theorem 3. The signal strength condition ($\|\mu\| \succsim 1$) is satisfied. The dimensionality conditions are derived by substituting the isotropic parameters into the $\Tr(\Sigma)$ requirements of Theorem 3.

The resulting condition on $p$ matches Corollary 3.6 of \citet{Hashimoto2025}. The inhomogeneous analysis also produces $p \succsim n^{3/2+2/k+1/l}$ (Term $B_{iso}$ in Appendix~\ref{sec:proof_cor1}), but under $r, k \in (2, 4]$ and $l \ge 2$ this is strictly weaker than the homogeneous requirement $p \succsim n^{1+8/r+2/l}$.

Details are in Appendix~\ref{sec:proof_cor1}.
\end{proof}

\section{Discussion}
Benign overfitting still holds with a bias term, but the assumptions change. The sample-size lower bounds are essentially the same as in the homogeneous case up to constants. What changes is the requirement on $\Tr(\Sigma)$. Table~\ref{tab:comparison} summarizes the comparison.

\subsection{Comparison of Assumptions and Interpretation}

\begin{table}[h!]
    \centering
    \begin{threeparttable}
        \caption{Dimensionality requirements ($\Tr(\Sigma)$) for benign overfitting, with and without a bias term. $H_1, H_2$ are the isotropic homogeneous thresholds\tnote{\textdagger}.}
        \label{tab:comparison}
        \begin{tabular}{@{}lll@{}}
            \toprule
            \textbf{Setting} & \textbf{Homogeneous Requirement} & \textbf{Additional Inhomogeneous Requirement} \\
            \midrule
            Noiseless, Int. Signal (Thm \ref{thm:1}) & $\Tr(\Sigma) \succsim T_{Hom}$ & $\Tr(\Sigma) \succsim n^{3/2}\left(\frac{n}{\delta}\right)^{2/k+1/l}$ \\
            \addlinespace
            Noiseless, Large Signal (Thm \ref{thm:2}) & $\Tr(\Sigma) \succsim T_{Hom}$ & None \\
            \addlinespace
            Noisy (Thm \ref{thm:3}) & $\Tr(\Sigma) \succsim T'_{Hom}$ & $\Tr(\Sigma) \succsim \frac{1}{\eta} n^{3/2}\left(\frac{n}{\delta}\right)^{2/k+1/l}$ \\
            \addlinespace
            Isotropic (Cor \ref{cor:1}) & $p \succsim \max\{H_1, H_2\}$ & None (New term is dominated) \\
            \bottomrule
        \end{tabular}
        \begin{tablenotes}
            \footnotesize
            \item[\textdagger] Where $H_1 = n^\frac{4+(1+2/l)r}{2(r-2)}$ and $H_2 = n^{1+8/r+2/l}$.
        \end{tablenotes}
    \end{threeparttable}
\end{table}

Sample-size requirements are essentially unchanged up to constants. The main difference is in $\Tr(\Sigma)$: the inhomogeneous model adds
\[
\Tr(\Sigma) \succsim n^{3/2}(n/\delta)^{2/k+1/l}.
\]
With label noise (Theorem~\ref{thm:3}) this is multiplied by $1/\eta$, so high noise forces a larger effective dimension. In the noiseless intermediate-signal regime (Theorem~\ref{thm:1}) the same term appears without the $1/\eta$ factor. In the isotropic case (Corollary~\ref{cor:1}) it is dominated by the homogeneous conditions, so the intercept does not change the asymptotic threshold.

\paragraph{Why the bias raises the dimension.}
Appending a constant coordinate to each $\boldsymbol{z}_i$ introduces a shared component across samples. The resulting off-diagonal perturbation $P_{ij}$ (Appendix~\ref{sec:E1}) pushes the normalized Gram matrix away from the identity. Near-orthogonality is restored only if $\Tr(\Sigma)$ is large enough relative to $n$ to dilute that shared correlation.

\paragraph{Noiseless case.}
Theorem~\ref{thm:1} also needs $\Tr(\Sigma) \succsim (n/\delta)^{4/k}$ to control the $\beta'$ error in $\tilde{E}_4$ coming from the shift $1/\|\boldsymbol{z}_i\|^2 \to 1/\|\tilde{\boldsymbol{z}}_i\|^2$ (Appendix~\ref{sec:E4E5}). For large $n$ this is weaker than the Gram-matrix condition $\Tr(\Sigma) \succsim n^{3/2+2/k+1/l}$ from $\tilde{E}_1$.

\paragraph{Noisy case.}
Under label noise the same Gram-matrix error $\tilde{\varepsilon}$ must be small relative to $\eta$ (condition $(N_C)$), which produces the $T'_{Inhom}$ term of order $n^{3/2+2/k+1/l}/\eta$. In the noiseless setting it is enough that $\tilde{\varepsilon}$ be a fixed constant (e.g., $1/4$), so the dimension requirement is milder.

\paragraph{Role of $k$.}
The new bounds depend on $k$ through negative moments of $g$. The bias perturbation involves factors such as $1/\|\boldsymbol{z}_i\|^2$ (Lemmas~\ref{lem:E1} and~\ref{lem:E4E5}), which are large when some noise norms are small. Smaller $k$ (closer to $2$) makes small norms more likely and therefore forces a larger $\Tr(\Sigma)$.

\subsection{Implications}
A bias term does not destroy benign overfitting, but it can make the sufficient conditions stricter: one needs larger $\Tr(\Sigma)$ relative to $n$ to keep the noise nearly orthogonal. Whether that new constraint binds depends on the covariance. Under isotropy it is often already implied by the homogeneous theory, so omitting the intercept does not change the asymptotic picture. Away from isotropy, or when $\eta$ is not tiny, the extra $T_{Inhom}$ / $T'_{Inhom}$ terms can become the binding constraint. All of our conditions are sufficient, not necessary. The extra factors may be an artifact of how we bound the perturbations rather than a sharp necessity.

\subsection{Limitations and Future Work}
As in \citet{Hashimoto2025}, we treat only the maximum margin classifier under Model EM. The model is broad but still a two-component mixture, and our arguments inherit that structure through the reduction to their Theorems 3.2 and 3.5.

Natural next steps include square-loss / regression analogues \citep{Bartlett2020}, and the role of an intercept in nonlinear models or other data-generating processes.

\section{Conclusion}
We extended the benign-overfitting analysis of \citet{Hashimoto2025} to linear classifiers with a bias term. Interpolation remains benign under explicit conditions on the signal and on $\Tr(\Sigma)$. Relative to the homogeneous case, the intercept adds dimensionality constraints that are sharpest with label noise. Under isotropic covariance those constraints are dominated by existing ones; under anisotropic covariance they can tighten the requirements.

\clearpage
\bibliographystyle{plainnat}
\bibliography{references}

\clearpage
\appendix

\section{Appendix}

This appendix contains the proofs. Appendix~\ref{sec:new_events} relates the homogeneous and inhomogeneous events; Appendix~\ref{sec:proofs_main_theorems} proves the main theorems; Appendix~\ref{sec:supplementary} collects remaining arguments.

We follow the notation of \citet{Hashimoto2025}. Write $\boldsymbol{z}_i = g_i \Sigma^{1/2} \boldsymbol{\xi}_i = g_i \boldsymbol{v}_i$ with $\boldsymbol{v}_i = \Sigma^{1/2} \boldsymbol{\xi}_i$. As in that paper, define
\[
\Omega_1(\varepsilon) := \left\{ \forall i \in \{1..n\}: (1-\varepsilon/4)\Tr(\Sigma) \le \|\boldsymbol{v}_i\|^2 \le (1+\varepsilon/4)\Tr(\Sigma) \right\}.
\]

\subsection{Analysis of Inhomogeneous Events}
\label{sec:new_events}

\subsubsection{Event $\tilde{E}_1$: Concentration of the Normalized Gram Matrix}
\label{sec:E1}

The event $\tilde{E}_1(\tilde{\varepsilon})$ is defined as $\big\{\|\check{\tilde{Z}}\check{\tilde{Z}}^\top - I_n\|\leq \tilde{\varepsilon}\big\}$. The analysis requires careful control of the noise norms. We first establish a prerequisite concentration event.

\begin{lemma}
\label{lem:E1}
(Requires Model (EM)). Suppose the events $E_1(\varepsilon)$ and $\Omega_1(\varepsilon)$ hold. Let $\delta_{E1} > 0$. Then, with probability at least $1 - \delta_{E1}$ (conditioned on $E_1, \Omega_1$), the event $\tilde{E}_1(\tilde{\varepsilon})$ holds with
    \[
    \tilde{\varepsilon} = \sqrt{n(n-1)} \cdot \left(\frac{n \cdot \mathbb{E}[g^{-k}]}{\delta_{E1}}\right)^{2/k} \left( \frac{2\varepsilon + 1}{(1-\varepsilon/4)\Tr(\Sigma)}\right) + \varepsilon.
    \]
\end{lemma}

\begin{proof}
We begin by decomposing the spectral norm using the triangle inequality:
\[
\left\| \check{\tilde{Z}}\check{\tilde{Z}}^T - I_n \right\| \leq \left\| \check{\tilde{Z}}\check{\tilde{Z}}^T - \check{Z}\check{Z}^T \right\| + \left\| \check{Z}\check{Z}^T - I_n \right\|.
\]
By the assumption $E_1(\varepsilon)$, defined as $\big\{\|\check{Z}\check{Z}^\top - I_n\|\leq \varepsilon\big\}$, the second term is bounded by $\varepsilon$. We define the perturbation matrix $P = \check{\tilde{Z}}\check{\tilde{Z}}^T-\check{Z}\check{Z}^T$. We aim to bound the spectral norm $\|P\|$. We use the property that the spectral norm is bounded by the Frobenius norm, $\|P\| \leq \|P\|_F$.

We analyze the elements of $P$. The diagonal elements satisfy $P_{ii}=0$.
For the off-diagonal elements ($i \neq j$):
\begin{align*}
P_{ij} &= \frac{\langle \boldsymbol{z}_i, \boldsymbol{z}_j \rangle}{\|\boldsymbol{z}_i\|\|\boldsymbol{z}_j\|} \left( \frac{\|\boldsymbol{z}_i\|\|\boldsymbol{z}_j\|}{\|\tilde{\boldsymbol{z}}_i\|\|\tilde{\boldsymbol{z}}_j\|} - 1 \right) + \frac{1}{\|\tilde{\boldsymbol{z}}_i\|\|\tilde{\boldsymbol{z}}_j\|}.
\end{align*}

Let $A_{ij} = \left| \frac{\|\boldsymbol{z}_i\|\|\boldsymbol{z}_j\|}{\|\tilde{\boldsymbol{z}}_i\|\|\tilde{\boldsymbol{z}}_j\|} - 1 \right|$ and $B_{ij} = \frac{1}{\|\tilde{\boldsymbol{z}}_i\|\|\tilde{\boldsymbol{z}}_j\|}$. By the triangle inequality and $E_1(\varepsilon)$ (which implies the off-diagonal entries of $\check{Z}\check{Z}^T$ are bounded by $\varepsilon$):
\[
|P_{ij}| \leq \varepsilon A_{ij} + B_{ij}.
\]

\paragraph{Step 1: Bounding $A_{ij}$}
We bound $A_{ij} = 1 - \frac{\|\boldsymbol{z}_i\|\|\boldsymbol{z}_j\|}{\|\tilde{\boldsymbol{z}}_i\|\|\tilde{\boldsymbol{z}}_j\|}$. Let $R_{ij} = \frac{\|\boldsymbol{z}_i\|^2\|\boldsymbol{z}_j\|^2}{(\|\boldsymbol{z}_i\|^2+1)(\|\boldsymbol{z}_j\|^2+1)}$. Then $A_{ij} = 1-\sqrt{R_{ij}}$. Since $R_{ij} \in [0, 1)$, we use the inequality $1-\sqrt{X} \le 1-X$.
\begin{align*}
A_{ij} \le 1-R_{ij} &= \frac{(\|\boldsymbol{z}_i\|^2+1)(\|\boldsymbol{z}_j\|^2+1) - \|\boldsymbol{z}_i\|^2\|\boldsymbol{z}_j\|^2}{(\|\boldsymbol{z}_i\|^2+1)(\|\boldsymbol{z}_j\|^2+1)} = \frac{\|\boldsymbol{z}_i\|^2+\|\boldsymbol{z}_j\|^2+1}{(\|\boldsymbol{z}_i\|^2+1)(\|\boldsymbol{z}_j\|^2+1)}.
\end{align*}
We use the inequality $\frac{a+b+1}{(a+1)(b+1)} < \frac{1}{a+1} + \frac{1}{b+1}$.
\begin{align*}
A_{ij} &\le \frac{1}{\|\boldsymbol{z}_i\|^2+1} + \frac{1}{\|\boldsymbol{z}_j\|^2+1} \le \|\boldsymbol{z}_i\|^{-2} + \|\boldsymbol{z}_j\|^{-2} \le 2 \max_k \frac{1}{\|\boldsymbol{z}_k\|^2}.
\end{align*}

\paragraph{Step 2: High Probability Bound on $\|\boldsymbol{z}_k\|^{-2}$}
We find a uniform bound for $\max_k \|\boldsymbol{z}_k\|^{-2}$. On the event $\Omega_1(\varepsilon)$, defined as $\left\{ \forall i: (1-\varepsilon/4)\Tr(\Sigma) \le \|\boldsymbol{v}_i\|^2 \le (1+\varepsilon/4)\Tr(\Sigma) \right\}$, we have $\max_k \|\boldsymbol{v}_k\|^{-2} \leq ((1-\varepsilon/4)\Tr(\Sigma))^{-1}$.

We bound $\max_k g_k^{-2}$. For $t>0$, we apply the union bound:
\[
\mathbb{P}\left(\max_k g_k^{-2} \ge t\right) \le n \cdot \mathbb{P}(g^{-2} \ge t) = n \cdot \mathbb{P}(g^{-k} \ge t^{k/2}).
\]
By Markov's inequality, which states that for a non-negative random variable $X$ and $a>0$, $\mathbb{P}(X \ge a) \le \mathbb{E}[X]/a$, we have:
\[
\mathbb{P}\left(\max_k g_k^{-2} \ge t\right) \le n \mathbb{E}[g^{-k}]/t^{k/2}.
\]
Setting this probability to $\delta_{E1}$ and solving for $t$:
\[
t = \left(\frac{n \cdot \mathbb{E}[g^{-k}]}{\delta_{E1}}\right)^{2/k}.
\]
Combining the bounds, with probability at least $1-\delta_{E1}$ (conditioned on $\Omega_1$):
\[
\max_k \frac{1}{\|\boldsymbol{z}_k\|^2} \leq \left(\frac{n \cdot \mathbb{E}[g^{-k}]}{\delta_{E1}}\right)^{2/k} \cdot \frac{1}{(1-\varepsilon/4)\Tr(\Sigma)} := T_{\text{bound}}.
\]

\paragraph{Step 3: Bounding $B_{ij}$ and $P_{ij}$}
Since $\|\tilde{\boldsymbol{z}}_k\| \ge \|\boldsymbol{z}_k\|$, $B_{ij} \le (\|\boldsymbol{z}_i\|\|\boldsymbol{z}_j\|)^{-1} \le \max_k \|\boldsymbol{z}_k\|^{-2} \le T_{\text{bound}}$.
Therefore,
$|P_{ij}| \leq \varepsilon A_{ij} + B_{ij} \le 2\varepsilon T_{\text{bound}} + T_{\text{bound}} = (2\varepsilon+1)T_{\text{bound}}.$

\paragraph{Step 4: Bounding $\|P\|_F$}
We calculate the Frobenius norm:
\begin{align*}
\|P\|_F^2 &= \sum_{i \neq j} |P_{ij}|^2 \leq n(n-1) ((2\varepsilon+1)T_{\text{bound}})^2.
\end{align*}
Thus, $\|P\| \le \|P\|_F \le \sqrt{n(n-1)}(2\varepsilon+1)T_{\text{bound}}$.
The total bound is $\tilde{\varepsilon} = \|P\| + \varepsilon$, completing the proof.
\end{proof}

\subsubsection{Event $\tilde{E}_2$: Alignment with Signal}
\label{sec:E2}
The event $\tilde{E}_2(\tilde{\alpha}_2, \tilde{\alpha}_\infty)$ concerns the alignment of the noise vectors with the signal $\tilde{\boldsymbol{\mu}}$.

\begin{lemma}
\label{lem:E2}
Suppose the event $E_2(\alpha_2, \alpha_\infty)$ holds. Then, the event $\tilde{E}_2(\alpha_2, \alpha_\infty)$ holds.
\end{lemma}

\begin{proof}
We compare the vector $\check{\tilde{Z}}\tilde{\boldsymbol{\mu}}$ with $\check{Z}\boldsymbol{\mu}$. First, note that $\|\tilde{\boldsymbol{\mu}}\| = \|(\boldsymbol{\mu}, 0)\| = \|\boldsymbol{\mu}\|$.

The $i$-th entry of the vector $\check{\tilde{Z}}\tilde{\boldsymbol{\mu}}$ is:
\[
(\check{\tilde{Z}}\tilde{\boldsymbol{\mu}})_i = \frac{\langle \tilde{\boldsymbol{z}}_i, \tilde{\boldsymbol{\mu}} \rangle}{\|\tilde{\boldsymbol{z}}_i\|} = \frac{\langle (\boldsymbol{z}_i, 1), (\boldsymbol{\mu}, 0) \rangle}{\sqrt{\|\boldsymbol{z}_i\|^2 + 1}} = \frac{\langle \boldsymbol{z}_i, \boldsymbol{\mu} \rangle}{\sqrt{\|\boldsymbol{z}_i\|^2 + 1}}.
\]
We compare this to the original entry, $(\check{Z}\boldsymbol{\mu})_i = \frac{\langle \boldsymbol{z}_i, \boldsymbol{\mu} \rangle}{\|\boldsymbol{z}_i\|}$.
\[
|(\check{\tilde{Z}}\tilde{\boldsymbol{\mu}})_i| = \left|\frac{\langle \boldsymbol{z}_i, \boldsymbol{\mu} \rangle}{\|\boldsymbol{z}_i\|}\right| \cdot \frac{\|\boldsymbol{z}_i\|}{\sqrt{\|\boldsymbol{z}_i\|^2+1}}.
\]
Since $\sqrt{\|\boldsymbol{z}_i\|^2+1} > \|\boldsymbol{z}_i\|$, the factor $\frac{\|\boldsymbol{z}_i\|}{\sqrt{\|\boldsymbol{z}_i\|^2+1}} < 1$. Therefore, $|(\check{\tilde{Z}}\tilde{\boldsymbol{\mu}})_i| < |(\check{Z}\boldsymbol{\mu})_i|$.

Since the inequality holds element-wise, it must also hold for any $L_p$ norm.
\begin{itemize}
    \item $L_\infty$ Norm: $\|\check{\tilde{Z}}\tilde{\boldsymbol{\mu}}\|_\infty < \|\check{Z}\boldsymbol{\mu}\|_\infty$.
    \item $L_2$ Norm: $\|\check{\tilde{Z}}\tilde{\boldsymbol{\mu}}\|^2 < \|\check{Z}\boldsymbol{\mu}\|^2$.
\end{itemize}
Given $E_2(\alpha_2, \alpha_\infty)$, defined as $\big\{\|\check{Z}\boldsymbol{\mu}\| \leq \alpha_2 \|\boldsymbol{\mu}\| \text{ and } \|\check{Z} \boldsymbol{\mu}\|_\infty \leq \alpha_\infty \|\boldsymbol{\mu}\|\big\}$, we have $\|\check{Z}\boldsymbol{\mu}\| \le \alpha_2\|\boldsymbol{\mu}\|$ and $\|\check{Z}\boldsymbol{\mu}\|_\infty \le \alpha_\infty\|\boldsymbol{\mu}\|$. Consequently, $\tilde{E}_2(\alpha_2, \alpha_\infty)$ holds.
\end{proof}

\subsubsection{Event $\tilde{E}_3$: Maximum Noise Norm}
\label{sec:E3}

The event $\tilde{E}_3(\tilde{M})$ bounds the maximum norm of the inhomogeneous noise vectors.

\begin{lemma}
\label{lem:E3}
Suppose the event $E_3(M)$ holds. Then, the event $\tilde{E}_3(\sqrt{M^2+1})$ holds.
\end{lemma}

\begin{proof}
The relationship between the norms is exact: $\|\tilde{\boldsymbol{z}}_i\| = \sqrt{\|\boldsymbol{z}_i\|^2+1}$.
We maximize this over all $i$:
\begin{align*}
\max_i \|\tilde{\boldsymbol{z}}_i\| &= \max_i \sqrt{\|\boldsymbol{z}_i\|^2 + 1} = \sqrt{(\max_i \|\boldsymbol{z}_i\|)^2 + 1}.
\end{align*}
By the assumption $E_3(M)$, defined as $\big\{\max_{i} \|\boldsymbol{z}_i\|\leq M\big\}$, we have $\max_i \|\boldsymbol{z}_i\| \le M$. Therefore, $\max_i \|\tilde{\boldsymbol{z}}_i\| \le \sqrt{M^2+1}$.
\end{proof}

\subsubsection{Events $\tilde{E}_4$ and $\tilde{E}_5$: Concentration of Inverse Norms}
\label{sec:E4E5}

The events $\tilde{E}_4$ and $\tilde{E}_5$ concern the concentration of averages involving the inverse squared norms $1/\|\tilde{\boldsymbol{z}}_i\|^2$. We analyze the concentration around the original $\rho = \mathbb{E}[g^{-2}]/\Tr(\Sigma)$.

\begin{lemma}
\label{lem:E4E5}
(Requires Model (EM)). Suppose the events $E_4(\beta, \rho)$ and $\Omega_1(\varepsilon)$ hold. Let $\delta_{E4} > 0$. Then, with probability at least $1-\delta_{E4}$, the event $\tilde{E}_4(\tilde{\beta}, \rho)$ holds with $\tilde{\beta} = \beta + \beta'$, where
    \[
    \beta' = \frac{1}{(1-\varepsilon/4)^2 \mathbb{E}[g^{-2}]} \cdot \frac{\left(\frac{n \cdot \mathbb{E}[g^{-k}]}{\delta_{E4}}\right)^{4/k}}{\Tr(\Sigma)}.
    \]
If, additionally, $E_5(\gamma, \rho)$ holds, then $\tilde{E}_5(\tilde{\gamma}, \rho)$ holds with $\tilde{\gamma} = \gamma + \beta'$.
\end{lemma}

\begin{proof}
We first analyze $\tilde{E}_4$. We decompose the deviation from the expectation $\rho$ using the triangle inequality:
\begin{align*}
\left|\frac{1}{n} \sum_{i=1}^n \frac{1}{\|\tilde{\boldsymbol{z}}_i\|^2} - \rho \right| &\leq \left|\frac{1}{n} \sum_{i=1}^n \left(\frac{1}{\|\tilde{\boldsymbol{z}}_i\|^2} - \frac{1}{\|\boldsymbol{z}_i\|^2}\right)\right| + \left|\frac{1}{n} \sum_{i=1}^n \frac{1}{\|\boldsymbol{z}_i\|^2} - \rho \right|.
\end{align*}
By $E_4(\beta, \rho)$, defined as $\Big\{ \Big|\frac{1}{n}\sum_{i} \frac{1}{\|\boldsymbol{z}_i\|^{2}}-\rho \Big|\leq \beta\rho \Big\}$, the second term is bounded by $\beta\rho$. We analyze the first term, the perturbation $B_{\text{pert}}$.
\begin{align*}
B_{\text{pert}} &= \frac{1}{n} \sum_{i=1}^n \left(\frac{1}{\|\boldsymbol{z}_i\|^2} - \frac{1}{\|\boldsymbol{z}_i\|^2+1}\right) = \frac{1}{n} \sum_{i=1}^n \frac{1}{\|\boldsymbol{z}_i\|^2(\|\boldsymbol{z}_i\|^2+1)} \\
&\le \frac{1}{n} \sum_{i=1}^n \frac{1}{\|\boldsymbol{z}_i\|^4} = \frac{1}{n} \sum_{i=1}^n \frac{1}{g_i^4 \|\boldsymbol{v}_i\|^4}.
\end{align*}

We need a high-probability bound on $\max_i (g_i^4 \|\boldsymbol{v}_i\|^4)^{-1}$.
On $\Omega_1(\varepsilon)$, $\max_i \|\boldsymbol{v}_i\|^{-4} \le ((1-\varepsilon/4)\Tr(\Sigma))^{-2}$.

We find a high-probability bound for $\max_i g_i^{-4}$. For $t>0$, we use the union bound:
\begin{align*}
\mathbb{P}\left(\max_{i} g_i^{-4} > t\right) &\leq n\, \mathbb{P}\left(g^{-4} > t\right) = n\, \mathbb{P}\left(g^{-k} > t^{k/4}\right).
\end{align*}
By Markov's inequality, we have:
\[
\mathbb{P}\left(\max_{i} g_i^{-4} > t\right) \leq n\, \mathbb{E}[g^{-k}] t^{-k/4}.
\]
Setting this probability to $\delta_{E4}$ and solving for $t$:
\[
t = \left(\frac{n\, \mathbb{E}[g^{-k}]}{\delta_{E4}}\right)^{4/k}.
\]
Thus, with probability $1-\delta_{E4}$ (conditioned on $\Omega_1$):
\[
B_{\text{pert}} \le \max_i \|\boldsymbol{z}_i\|^{-4} \le \left(\frac{n \mathbb{E}[g^{-k}]}{\delta_{E4}}\right)^{4/k} \frac{1}{(1-\varepsilon/4)^2 \Tr(\Sigma)^2}.
\]
We define this bound as $\rho\beta'$. Substituting $\rho = \mathbb{E}[g^{-2}]/\Tr(\Sigma)$ and solving for $\beta'$ yields the expression in the lemma.

Now we analyze $\tilde{E}_5$. Let $S_{\tilde{z},y} = \frac{1}{n}\sum \frac{y_{N,i}y_i}{\|\tilde{\boldsymbol{z}}_i\|^2}$ and $S_{z,y} = \frac{1}{n}\sum \frac{y_{N,i}y_i}{\|\boldsymbol{z}_i\|^2}$. By the triangle inequality:
\[
\left| S_{\tilde{z},y} - (1-2\eta)\rho \right| \le \left| S_{\tilde{z},y} - S_{z,y} \right| + \left| S_{z,y} - (1-2\eta)\rho \right|.
\]
The second term is bounded by $\gamma\rho$ by $E_5(\gamma, \rho)$, defined as $\Big\{\Big|\frac{1}{n}\sum_{i} \frac{y_{N, i} y_i}{\|\boldsymbol{z}_i\|^2} -(1-2\eta)\rho \Big|\leq \gamma\rho \Big\}$. The first term is bounded by $B_{\text{pert}}$ since $|y_{N,i}y_i|=1$. The total deviation is bounded by $(\gamma+\beta')\rho$.
\end{proof}

\subsection{Proofs of Main Theorems}
\label{sec:proofs_main_theorems}

We verify the hypotheses of Theorems 3.2 and 3.5 of \citet{Hashimoto2025} for the inhomogeneous model, using the lemmas above. The homogeneous event parameters are taken from their Lemma 3.1; we restate them for convenience.

\textbf{Definitions from Lemma 3.1 of \citet{Hashimoto2025}:}
Suppose model (EM) holds. Provided $\Tr(\Sigma)$ is sufficiently large such that $\varepsilon \le 1/2$, the events $E_1$ through $E_5$ hold simultaneously with high probability (specific probabilities discussed in the proofs below) with the following parameters, where $C_1(r,K)$ and $C_2(g)$ are constants depending on the distribution of $g$ and $\xi$:
\begin{align*}
\varepsilon &:= C_1(r, K) \left(\frac{n}{\delta}\right)^{2/r} \max\{p^{2/r-1/2}, n^{2/r}\} \frac{\|\Sigma\|_F}{\Tr(\Sigma)} \\
\alpha_\infty &= \alpha_2 := \frac{2\sqrt{n}\|\Sigma^{1/2}\mu\|}{\sqrt{\delta}\Tr(\Sigma) \|\mu\|} \\
M &:= (1+\varepsilon)\|g\|_{L^l}\left(\frac{n}{\delta}\right)^{1/l}\sqrt{\Tr(\Sigma)} \\
\rho &:= \mathbb{E}[g^{-2}]\Tr(\Sigma)^{-1} \\
\beta &= \gamma := \varepsilon + C_2(g)\delta^{-2/k}n^{-(1-2/k)}
\end{align*}

We set the failure probabilities $\delta_{E1}, \delta_{E4}$ from Appendix A equal to the overall failure tolerance $\delta$.

\subsubsection{Proof of Theorem \ref{thm:1} (Inhomogeneous, Noiseless, Intermediate Signal)}
\label{sec:proof_thm1}

\begin{proof}
The proof relies on applying Theorem 3.2(i) of \citet{Hashimoto2025} to the inhomogeneous variables ($\tilde{\boldsymbol{x}}, \tilde{\boldsymbol{z}}, \tilde{\boldsymbol{\mu}}$). We first restate the theorem's conditions adapted to our notation.

\textbf{Theorem 3.2(i) of \citet{Hashimoto2025} (Applied to Inhomogeneous Variables):}
The theorem states that if the events $\tilde{E}_1(\tilde{\varepsilon}), \tilde{E}_2(\tilde{\alpha}_2, \tilde{\alpha}_\infty), \tilde{E}_3(\tilde{M}), \tilde{E}_4(\tilde{\beta}, \rho)$ hold and the following conditions are met:
\begin{enumerate}
    \item[(i)] $\tilde{\beta}<1/2$;
    \item[(ii)] $\|\tilde{\boldsymbol{\mu}}\|\sqrt{(1-\tilde{\beta})n\rho} \geq C\tilde{\alpha_2}$ (for a sufficiently large constant $C$);
    \item[(iii)] $\tilde{\alpha_2}\|\tilde{\boldsymbol{\mu}}\|\sqrt{(1+\tilde{\beta})n\rho} \leq 1/4$;
    \item[(iv)] $\tilde{\varepsilon}\tilde{M}\sqrt{(1+\tilde{\beta})n\rho} \leq 1/4$;
    \item[(v)] $\tilde{M}\tilde{\alpha_\infty}\|\tilde{\boldsymbol{\mu}}\|(1+\tilde{\beta})n\rho < 3/32$.
\end{enumerate}
Then the maximum margin classifier $\hat{\tilde{\boldsymbol{w}}}$ satisfies the test error bound stated in Theorem 1.

We now verify these conditions.

\paragraph{Probability Accounting}
We first account for the failure probabilities. The assumptions $T_{Hom}$ ensure the original events $E_i$ and $\Omega_1$ hold with probability at least $1-3\delta$. This follows from adapting Lemma 3.1 of \citet{Hashimoto2025} for the noiseless case (where $E_5$ is not required). The subsequent analysis in Lemmas~\ref{lem:E1} (bounding $\tilde{\varepsilon}$) and \ref{lem:E4E5} (bounding $\tilde{\beta}$) introduces failure probabilities $\delta_{E1}=\delta$ and $\delta_{E4}=\delta$. By the union bound, the total failure probability is bounded by $3\delta + \delta + \delta = 5\delta$.

We use the relationships derived in Appendix A and the definitions of the homogeneous parameters (restated above).

\paragraph{(i) Show $\tilde{\beta} < 1/2$}
By Lemma~\ref{lem:E4E5}, we have $\tilde{\beta} = \beta + \beta'$. Substituting the definitions of $\beta$ (from Hashimoto et al. Lemma 3.1) and $\beta'$ (from Lemma~\ref{lem:E4E5}):
\begin{align*}
\tilde{\beta} = \underbrace{\left(\varepsilon + \CtwoG\delta^{-2/k}n^{-(1-2/k)}\right)}_{\beta} + \underbrace{\left(\frac{1}{(1-\varepsilon/4)^2\mathbb{E}[g^{-2}]} \cdot \frac{\left(\frac{n\mathbb{E}[g^{-k}]}{\delta}\right)^{4/k}}{\Tr(\Sigma)}\right)}_{\beta'}.
\end{align*}
We show that each of the three terms is less than $1/6$.

First Term ($\varepsilon$): The assumptions on $\Tr(\Sigma)$ ($T_{Hom}$), inherited from the requirements of the homogeneous analysis, ensure that $\varepsilon < 1/6$ (by choosing the constant $C$ in Theorem 1 large enough).

Second Term (Concentration of $g$): We analyze the condition on $n$ required for the concentration term in $\beta$.
\begin{align*}
& \CtwoG\delta^{-2/k}n^{-(1-2/k)} \leq \frac{1}{6} \\
\iff & n \geq \left(6\CtwoG\right)^\frac{k}{k-2} \delta^{-\frac{2}{k-2}}.
\end{align*}
This is satisfied by the assumption on $n$ in Theorem \ref{thm:1}.

Third Term (Perturbation $\beta'$): We require $\beta' \le 1/6$.
\begin{align*}
\impliedby & \Tr(\Sigma) \ge \frac{6}{(1-\varepsilon/4)^2\mathbb{E}[g^{-2}]} \cdot \left(\frac{n\mathbb{E}[g^{-k}]}{\delta}\right)^{4/k}.
\end{align*}
This requires $\Tr(\Sigma) \succsim (n/\delta)^{4/k}$. This condition is dominated by the requirement derived in paragraph (iv), $\Tr(\Sigma) \succsim n^{3/2}(n/\delta)^{2/k+1/l}$. The dominance holds because the exponent of $n$ in the latter requirement, $E_{B} = 3/2+2/k+1/l$, is strictly greater than the exponent in the former, $E_{A} = 4/k$. The difference $E_B - E_A = 3/2+1/l-2/k$. The infimum of this difference occurs as $k\to 2$ and $l\to\infty$, yielding $3/2+0-1 = 1/2$. Since the exponent difference is positive, the condition from step (iv) dominates for large $n$.

Thus, $\tilde{\beta} < 1/2$.

\paragraph{(ii) and (iii) Signal Strength Conditions}
Since $\|\tilde{\boldsymbol{\mu}}\|=\|\boldsymbol{\mu}\|$, $\tilde{\alpha}_2 \le \alpha_2$ (by Lemma~\ref{lem:E2}), and $\tilde{\beta} < 1/2$, these conditions are satisfied by the assumptions on $\|\boldsymbol{\mu}\|$ and $\Tr(\Sigma)$ ($T_{Hom}$) inherited from the homogeneous analysis (Theorem 3.3 of \citet{Hashimoto2025}), potentially requiring a slightly adjusted constant $C$.

\paragraph{(iv) Show $\tilde{\varepsilon}\tilde{M}\sqrt{(1+\tilde{\beta})n\rho} \leq 1/4$}
By Lemma~\ref{lem:M_bound} (Appendix~\ref{sec:aux_lemmas}), $M\ge 1$ (for sufficiently large $n$). By Lemma~\ref{lem:E3}, $\tilde{M} = \sqrt{M^2+1} \le \sqrt{2}M$.
Let $T$ be the perturbation term in $\tilde{\varepsilon}$ from Lemma~\ref{lem:E1}. Then $\tilde{\varepsilon} \le T + \varepsilon$.

LHS $\le (\varepsilon+T)\sqrt{2}M\sqrt{(1+\tilde{\beta})n\rho}$.

The term $\varepsilon\sqrt{2}M\sqrt{(1+\tilde{\beta})n\rho}$ is bounded by $1/8$ due to the conditions inherited from the homogeneous case ($T_{Hom}$), which ensure condition (iv) of Theorem 3.2(i) of \citet{Hashimoto2025} holds for the homogeneous parameters.

We must verify $T\sqrt{2}M\sqrt{(1+\tilde{\beta})n\rho} \le 1/8$.
We substitute the definitions of $T$, $M$ (from Lemma 3.1 of \citet{Hashimoto2025}, $M=(1+\varepsilon)\|g\|_{L^l}(\frac{n}{\delta})^{1/l}\sqrt{\Tr(\Sigma)}$), and $\rho=\mathbb{E}[g^{-2}]/\Tr(\Sigma)$.
We analyze the scaling (ignoring constants and $\varepsilon$ terms):
\begin{align*}
T \cdot M \cdot \sqrt{n\rho} &\approx \left(n \cdot \left(\frac{n}{\delta}\right)^{2/k} \cdot \frac{1}{\Tr(\Sigma)}\right) \cdot \left(\left(\frac{n}{\delta}\right)^{1/l}\sqrt{\Tr(\Sigma)}\right) \cdot \left(\sqrt{n/\Tr(\Sigma)}\right) \\
&= \frac{n^{3/2}}{\Tr(\Sigma)} \left(\frac{n}{\delta}\right)^{2/k+1/l}.
\end{align*}
The condition requires $\Tr(\Sigma) \succsim n^{3/2}(n/\delta)^{2/k+1/l}$. This requirement is satisfied by the $T_{Inhom}$ assumptions of Theorem \ref{thm:1}.

\paragraph{(v) Show $\tilde{M}\tilde{\alpha_\infty}\|\tilde{\boldsymbol{\mu}}\|(1+\tilde{\beta})n\rho < 3/32$}
Since $\tilde{M} \le \sqrt{2}M$ and $\tilde{\alpha}_\infty \le \alpha_\infty$ (by Lemma~\ref{lem:E2}), this condition is satisfied by the assumptions on $\Tr(\Sigma)$ inherited from the homogeneous case ($T_{Hom}$).
\end{proof}

\subsubsection{Proof of Theorem 2 (Inhomogeneous, Noiseless, Large Signal)}
\label{sec:proof_thm2}

\begin{proof}
We verify the conditions of Theorem 3.2(ii) of \citet{Hashimoto2025} applied to the inhomogeneous variables.

\textbf{Theorem 3.2(ii) of \citet{Hashimoto2025} (Applied to Inhomogeneous Variables):}
The theorem states that if the event $\tilde{E}_3(\tilde{M})$ holds and $\|\tilde{\boldsymbol{\mu}}\| \ge C_H\tilde{M}$ for some constant $C_H>2$, then the test error bound stated in Theorem 2 holds.

We verify these conditions.

\paragraph{Verification of $\tilde{E}_3(\tilde{M})$}
The High Dimension assumption ensures that $\varepsilon \le 1/2$ (where $\varepsilon$ is defined in Lemma 3.1 of \citet{Hashimoto2025}). By Lemma 3.1 of \citet{Hashimoto2025}, $E_3(M)$ holds with probability $1-2\delta$ (accounting for the probability of the prerequisite event $\Omega_1$), where $M$ is defined as:
\[ M=(1+\varepsilon)\|g\|_{L^l}\left(\frac{n}{\delta}\right)^{1/l}\sqrt{\Tr(\Sigma)}. \]

By Lemma~\ref{lem:E3}, $\tilde{E}_3(\tilde{M})$ holds with $\tilde{M}=\sqrt{M^2+1}$.

\paragraph{Verification of Signal Strength}
By Lemma~\ref{lem:M_bound} (Appendix~\ref{sec:aux_lemmas}), $M\ge 1$ (for sufficiently large $n$), so $\tilde{M} \le \sqrt{2}M$.

We require $\|\tilde{\boldsymbol{\mu}}\| \ge C_H\tilde{M}$. Since $\|\tilde{\boldsymbol{\mu}}\|=\|\boldsymbol{\mu}\|$, this condition is implied by $\|\boldsymbol{\mu}\| \ge C_H\sqrt{2}M$. Substituting the definition of $M$ and using the bound $\varepsilon \le 1/2$:
\begin{align*}
\|\boldsymbol{\mu}\| &\ge C_H\sqrt{2} (1+\varepsilon)\|g\|_{L^l}\left(\frac{n}{\delta}\right)^{1/l}\sqrt{\Tr(\Sigma)} \\
&\impliedby \|\boldsymbol{\mu}\| \ge C_H\sqrt{2} \cdot \frac{3}{2}\|g\|_{L^l}\left(\frac{n}{\delta}\right)^{1/l}\sqrt{\Tr(\Sigma)}.
\end{align*}
This is exactly the Large Signal condition of Theorem 2.

\paragraph{Conclusion}
The test error bound follows directly from Theorem 3.2(ii) of \citet{Hashimoto2025}. In this regime, the bound simplifies to $O(\|\mathbb{E}[\tilde{\boldsymbol{z}}\tilde{\boldsymbol{z}}^{\top}]\|/\|\tilde{\boldsymbol{\mu}}\|^2)$. Note that $\mathbb{E}[\tilde{\boldsymbol{z}}\tilde{\boldsymbol{z}}^{\top}] = \text{diag}(\Sigma, 1)$, so its spectral norm is $\max(\|\Sigma\|, 1)$. The total failure probability is at most $2\delta$.
\end{proof}

\subsubsection{Proof of Theorem 3 (Inhomogeneous, Noisy)}
\label{sec:proof_thm3}

\begin{proof}
The proof relies on verifying the conditions of Theorem 3.5 of \citet{Hashimoto2025} applied to the inhomogeneous variables.

\textbf{Theorem 3.5 of \citet{Hashimoto2025} (Applied to Inhomogeneous Variables):}
The theorem states that if $\eta \in (0, 1/2)$, and the following conditions hold:
\begin{enumerate}
    \item[(1)] $\tilde{\varepsilon} \vee \tilde{\beta} \vee \tilde{\gamma} \le \frac{\min\{\eta, 1-2\eta\}}{8}$.
    \item[(2)] Condition $(N_C)$ holds (adapted for inhomogeneous variables), which requires $\bigcap_{i=1}^5 \tilde{E}_i$ holds with $\tilde{\varepsilon} \tilde{M} \sqrt{n\rho} \le \eta/2$, $\|\tilde{\boldsymbol{\mu}}\| \ge C \sqrt{\tilde{\alpha}_2/(n\rho)}$ (where $C$ depends on $\eta$), and one of the two signal regimes (i) or (ii) defined in the theorem statement holds.
\end{enumerate}
Then the maximum margin classifier $\hat{\tilde{\boldsymbol{w}}}$ satisfies the test error bound in Theorem 3.

We now verify these conditions.

\paragraph{Probability Accounting}
The assumptions $T'_{Hom}$ ensure the original events $E_i$ and $\Omega_1$ hold with probability at least $1-4\delta$. This follows from adapting Lemma 3.1 of \citet{Hashimoto2025} for the noisy case. The subsequent analysis in Lemmas~\ref{lem:E1} (bounding $\tilde{\varepsilon}$) and \ref{lem:E4E5} (bounding $\tilde{\beta}, \tilde{\gamma}$) introduces failure probabilities $\delta_{E1}=\delta$ and $\delta_{E4}=\delta$. By the union bound, the total failure probability is bounded by $4\delta + \delta + \delta = 6\delta$.

\paragraph{(1) Bounding $\tilde{\beta}$ and $\tilde{\gamma}$}
By Lemma 3.1 of \citet{Hashimoto2025}, the homogeneous parameters satisfy $\beta=\gamma$. By Lemma~\ref{lem:E4E5}, the inhomogeneous parameters satisfy $\tilde{\beta}=\tilde{\gamma}$. We need $\tilde{\beta} = \beta+\beta' \le \frac{\min\{\eta, 1-2\eta\}}{8}$.

We require the components of $\beta$ and $\beta'$ to be small.
The assumptions on $\Tr(\Sigma)$ ($T'_{Hom}$) ensure $\varepsilon \le \frac{\min\{\eta, 1-2\eta\}}{32}$ (by appropriate choice of the constant $C$ in Theorem 3).

The condition on $n$: $n \ge \delta^{-\frac{2}{k-2}}(\frac{32\CtwoG}{\min\{\eta, 1-2\eta\}})^\frac{k}{k-2}$ ensures the concentration term of $\beta$ (defined in Lemma 3.1 of \citet{Hashimoto2025}), $\CtwoG\delta^{-2/k}n^{-(1-2/k)}$, is $\le \frac{\min\{\eta, 1-2\eta\}}{32}$.

We require the perturbation $\beta' \le \frac{\min\{\eta, 1-2\eta\}}{16}$. Based on the definition of $\beta'$ in Lemma~\ref{lem:E4E5}, this implies a condition on $\Tr(\Sigma)$:
\[
\Tr(\Sigma) \ge C' \frac{1}{\min\{\eta, 1-2\eta\}} \left(\frac{n}{\delta}\right)^{4/k}.
\]
This is satisfied by the assumptions of Theorem 3, as this requirement is dominated by the $T'_{Inhom}$ term for sufficiently large $n$ (since $3/2+1/l-2/k > 1/2$).

\paragraph{(2) Bounding $\tilde{\varepsilon}$}
We need $\tilde{\varepsilon} \le \frac{\min\{\eta, 1-2\eta\}}{8}$. We use the bound $\tilde{\varepsilon} \le T + \varepsilon$ (Lemma~\ref{lem:E1}).
We know $\varepsilon$ is small enough. We need the perturbation term $T \le \frac{\min\{\eta, 1-2\eta\}}{16}$.
$T \approx n (n/\delta)^{2/k} / \Tr(\Sigma)$ (ignoring constants). This requires:
\[
\Tr(\Sigma) \succsim \frac{n(n/\delta)^{2/k}}{\min\{\eta, 1-2\eta\}}.
\]
This is satisfied by the assumptions of Theorem 3, as this requirement is dominated by the $T'_{Inhom}$ term for sufficiently large $n$ (since $1/2+1/l > 0$).

\paragraph{(3) Verifying Condition $(N_C)$}
The main requirement of $(N_C)$ is $\tilde{\varepsilon} \tilde{M} \sqrt{n\rho} \le \eta/2$. We analyze $\tilde{\varepsilon} \le T + \varepsilon$.
The term $\varepsilon \tilde{M} \sqrt{n\rho}$ is controlled by the conditions inherited from the homogeneous analysis ($T'_{Hom}$).

We must control the perturbation term $T \tilde{M} \sqrt{n\rho}$. As calculated in the proof of Theorem \ref{thm:1} (Appendix~\ref{sec:proof_thm1}), using $\tilde{M}\le \sqrt{2}M$ and the definitions of $T, M, \rho$:
\[
T \tilde{M} \sqrt{n\rho} \approx C' \frac{n^{3/2}(n/\delta)^{2/k+1/l}}{\Tr(\Sigma)}.
\]

We require this to be $\le \eta/2$ (up to constants). This leads to the condition:
\[
\Tr(\Sigma) \succsim \frac{1}{\eta} n^{3/2}(n/\delta)^{2/k+1/l}.
\]
This corresponds to the $T'_{Inhom}$ assumption of Theorem 3.

The remaining parts of $(N_C)$ (signal strength requirement on $\|\boldsymbol{\mu}\|$ and the regime conditions (i) or (ii)) are satisfied by the corresponding assumptions in Theorem 3, similar to the analysis in the homogeneous case (Theorem 3.7 of \citet{Hashimoto2025}), adjusting constants to account for $\tilde{M} \le \sqrt{2}M$.
\end{proof}

\subsection{Supplementary Material}
\label{sec:supplementary}

\subsubsection{Proof of Corollary 1 (Isotropic Case)}
\label{sec:proof_cor1}

\begin{proof}
We verify the conditions of Theorem 3 in the isotropic case, where $\Sigma=I_p$. We substitute the isotropic parameters: $\Tr(\Sigma)=p$, $\|\Sigma\|_F=\sqrt{p}$, $\|\Sigma^{1/2}\boldsymbol{\mu}\|=\|\boldsymbol{\mu}\|$, and $\|\Sigma\|=1$. We analyze the asymptotics in $n, p, \|\boldsymbol{\mu}\|$, treating constants (including $\eta, \delta, C$, and distributional parameters $k, l, r$) as $O(1)$.

\paragraph{Test Error Bound and Benign Overfitting Condition}
The test error bound from Theorem 3 simplifies (using $\rho = \mathbb{E}[g^{-2}]/\Tr(\Sigma) \approx 1/p$) to:
\[
\mathbb{P}(\langle\hat{\tilde{\boldsymbol{w}}}, y\boldsymbol{\tilde{x}}\rangle < 0) - \eta \le c'\left(\eta \frac{n}{p} + \frac{1}{\|\boldsymbol{\mu}\|^2} + \frac{p}{n\|\boldsymbol{\mu}\|^4}\right).
\]
For benign overfitting, the Right Hand Side (RHS) must vanish as $n, p \to \infty$. This requires $n/p \to 0$ and $p/(n\|\boldsymbol{\mu}\|^4) \to 0$, which is equivalent to $\|\boldsymbol{\mu}\| \gg (p/n)^{1/4}$.

\paragraph{Verification of Theorem 3 Conditions}
Now we verify the conditions of Theorem 3.
1. Signal Strength: The first condition of Theorem 3 simplifies to $\|\boldsymbol{\mu}\|^2 \succsim \|\boldsymbol{\mu}\|$, which means $\|\boldsymbol{\mu}\| \succsim 1$. This is satisfied as we require $\|\boldsymbol{\mu}\| \to \infty$ for benign overfitting.

2. Dimensionality (Trace condition): We require $p \succsim \max\{T'_{Hom}, T'_{Inhom}\}$. We analyze the isotropic scaling of these terms, denoted $A_{iso}$ and $B_{iso}$ respectively.

Term $A_{iso}$ (from $T'_{Hom}$): $T'_{Hom} \succsim n^{1/2+2/r+1/l} \max\{p^{2/r-1/2}, n^{2/r}\} \sqrt{p}$.
We require $p \succsim A_{iso}$.

Case A1 ($p^{2/r-1/2} \ge n^{2/r}$): $p \succsim n^{1/2+2/r+1/l} p^{2/r-1/2} \sqrt{p} = n^{1/2+2/r+1/l} p^{2/r}$.
Rearranging gives $p^{1-2/r} \succsim n^{1/2+2/r+1/l}$. Since $r>2$, $1-2/r>0$. This yields the condition:
\[ p \succsim n^{\frac{1/2+2/r+1/l}{1-2/r}} = n^{\frac{4+(1+2/l)r}{2(r-2)}}. \quad \text{(Exponent } E_1 \text{)} \]

Case A2 ($p^{2/r-1/2} < n^{2/r}$): $p \succsim n^{1/2+2/r+1/l} n^{2/r} \sqrt{p} = n^{1/2+4/r+1/l} \sqrt{p}$.
Rearranging gives $\sqrt{p} \succsim n^{1/2+4/r+1/l}$. This yields the condition:
\[ p \succsim n^{1+8/r+2/l}. \quad \text{(Exponent } E_2 \text{)} \]

Term $B_{iso}$ (from $T'_{Inhom}$): $T'_{Inhom} \succsim n^{3/2+2/k+1/l}$.
We require $p \succsim B_{iso}$.
\[ p \succsim n^{3/2+2/k+1/l}. \quad \text{(Exponent } E_3 \text{)} \]

We compare the exponents $E_2$ (homogeneous requirement) and $E_3$ (inhomogeneous requirement).
\[ E_2 - E_3 = (1+8/r+2/l) - (3/2+2/k+1/l) = 8/r - 2/k + 1/l - 1/2. \]
We find the infimum of this difference under the constraints $r, k \in (2, 4], l \ge 2$. We minimize the positive terms and maximize the negative terms: $r=4$ (minimizing $8/r$), $l\to\infty$ (minimizing $1/l$), and $k\to 2$ (maximizing $2/k$).
\[ \inf(E_2 - E_3) = 8/4 - 2/2 + 0 - 1/2 = 2 - 1 - 0.5 = 0.5. \]
Since the difference is strictly positive, $E_2$ dominates $E_3$. The new inhomogeneous requirement $B_{iso}$ is therefore redundant.

The combined dimensionality requirement is:
\[
p \succsim \max\left\{n^\frac{4+(1+2/l)r}{2(r-2)}, n^{1+8/r+2/l}\right\}.
\]

3. Regime Conditions: We need condition (1) OR (2) from Theorem 3 to hold.
(1) $p \succsim n^{3/2+1/l} \|\boldsymbol{\mu}\|$.
(2) $\|\boldsymbol{\mu}\| \succsim n^{1/2+1/l}$ AND $p \succsim n^{3/2+1/l}$.

The derived condition on $p$ (Exponent $E_2$) ensures $p \gg n^{3/2+1/l}$ (since $1+8/r+2/l > 3/2+1/l$ as $r>2$).

If $\|\boldsymbol{\mu}\| \succsim n^{1/2+1/l}$, regime (2) holds.

If $1 \precsim \|\boldsymbol{\mu}\| \prec n^{1/2+1/l}$, we check regime (1). We need $p \succsim n^{3/2+1/l} \|\boldsymbol{\mu}\|$. Since $\|\boldsymbol{\mu}\| < n^{1/2+1/l}$, a sufficient condition is $p \succsim n^{3/2+1/l} n^{1/2+1/l} = n^{2+2/l}$.
The derived condition on $p$ (Exponent $E_2$) ensures this, since $1+8/r+2/l > 2+2/l$ (as $r<8$, which is satisfied since $r \le 4$).

Thus, the derived conditions on $p$ (which imply $p\gg n$) and $\|\boldsymbol{\mu}\| \gg (p/n)^{1/4}$ are sufficient for benign overfitting.
\end{proof}

\subsubsection{Auxiliary Lemmas}
\label{sec:aux_lemmas}

\begin{lemma}[Bound on M]
\label{lem:M_bound}
Under the conditions of the main theorems (which imply the conditions of Lemma 3.1 of \citet{Hashimoto2025}), for sufficiently large $n$, the high-probability bound $M$ on the noise norm $\|\boldsymbol{z}_i\|$ is greater than 1.
\end{lemma}
\begin{proof}
The bound $M$ is defined in Lemma 3.1 of \citet{Hashimoto2025} as $M = (1+\varepsilon)\|g\|_{L^\ell}\left(\frac{n}{\delta}\right)^{1/l}\sqrt{\Tr(\Sigma)}$.

We analyze the definition of $M$:
\begin{itemize}[noitemsep,topsep=0pt]
    \item The conditions ensure $\varepsilon > 0$, so $(1+\varepsilon) > 1$.
    \item Since $l\ge 2$ and $\mathbb{E}[g^2]=1$. By Jensen's inequality (specifically, Lyapunov's inequality, which states $\|X\|_{L^s} \le \|X\|_{L^t}$ for $s \le t$), we have $\|g\|_{L^\ell} \ge \|g\|_{L^2} = \sqrt{\mathbb{E}[g^2]} = 1$.
    \item The conditions require $n \ge 1$ and $\delta$ typically small (e.g., $\delta < 1$), so $(n/\delta) > 1$. Since $l\ge 2$, $(n/\delta)^{1/l} > 1$.
    \item The conditions of the main theorems require $\Tr(\Sigma)$ to grow polynomially with $n$. For instance, in Theorem \ref{thm:1}, $T_{Inhom}$ requires $\Tr(\Sigma) \succsim n^{3/2+2/k+1/l}$. For sufficiently large $n$, $\sqrt{\Tr(\Sigma)} > 1$.
\end{itemize}
Since $M$ is a product of terms greater than or equal to 1, we must have $M \ge 1$.
\end{proof}

\end{document}